\documentclass[letter, 10pt, conference]{ieeeconf}      
\IEEEoverridecommandlockouts
\usepackage[utf8]{inputenc}

\usepackage{amsmath,amsthm,amsfonts,amssymb,bm,accents,dsfont,bbm}
\usepackage{graphicx,cite,enumerate,afterpage}
\usepackage{url}
\graphicspath{{Images/}}

\usepackage{tikz}
\usetikzlibrary{arrows,arrows.meta}
\usepackage{pgfplots}

\usepackage{blindtext}

\usepackage{color}

\DeclareMathOperator{\diag}{diag}

\newcommand{\norm}[1]{\ensuremath{\left\| #1 \right\|}}
\newcommand{\abs}[1]{\ensuremath{{\left\vert #1 \right\vert}}}

\DeclareFontFamily{U}{mathb}{\hyphenchar\font45}
\DeclareFontShape{U}{mathb}{m}{n}{
<-6> mathb5 <6-7> mathb6 <7-8> mathb7
<8-9> mathb8 <9-10> mathb9
<10-12> mathb10 <12-> mathb12
}{}
\DeclareSymbolFont{mathb}{U}{mathb}{m}{n}
\DeclareMathSymbol{\llcurly}{\mathrel}{mathb}{"CE}
\DeclareMathSymbol{\ggcurly}{\mathrel}{mathb}{"CF}

\newcommand{\calS}{\ensuremath{\mathcal{S}}}

\newcommand{\ccalA}{\ensuremath{\mathcal{A}}}

\newcommand{\ccalL}{\ensuremath{\mathcal{L}}}

\newcommand{\ccalS}{\ensuremath{\mathcal{S}}}

\newcommand{\ba}{\ensuremath{\bm{a}}}

\newcommand{\bs}{\ensuremath{\bm{s}}}

\newcommand{\bbeta}{\ensuremath{\bm{\beta}}}

\newcommand{\bbR}{\ensuremath{\mathbb{R}}}

\newcommand{\bbX}{\ensuremath{\mathbb{X}}}

\newcommand{\fkp}{\ensuremath{\mathfrak{p}}}

\def\st/{\textsuperscript{st}}
\def\nd/{\textsuperscript{nd}}
\def\rd/{\textsuperscript{rd}}
\def\th/{\textsuperscript{th}}

\newcommand{\setR}{\bbR}

\makeatletter
\def\nnil{\nil}
\newcounter{prob}

\newcounter{dual}

\makeatother

\newenvironment{prob*}{%
	\csname equation*\endcsname%
	\aligned%
}{%
	\endaligned%
	\csname endequation*\endcsname%
}

\newcommand{\executeiffilenewer}[3]{%
\ifnum\pdfstrcmp{\pdffilemoddate{#1}}%
{\pdffilemoddate{#2}}>0%
{\immediate\write18{#3}}\fi%
}

\usepackage[english]{babel}

\usepackage{algpseudocode,algorithm}
\algrenewcommand\algorithmicdo{}
\algrenewtext{EndFor}{\algorithmicend}
\algrenewtext{EndProcedure}{\algorithmicend}

\newtheorem{theorem}{Theorem}
\newtheorem{proposition}{Proposition}

\newtheorem{lemma}{Lemma}
\theoremstyle{definition}
\newtheorem{definition}{Definition}

\newtheoremstyle{assume}
  {3pt}
  {3pt}
  {}
  {}
  {\bf}
  {}
  { }
  {\thmname{#1}.\thmnumber{#2}\thmnote{ \textnormal{(\textit{#3})}}}
\theoremstyle{assume}

\author{Miguel Calvo-Fullana$^{1}$, Luiz F. O. Chamon$^{2}$, and Santiago Paternain$^{3}$
\thanks{$^{1}$Department of Aeronautics and Astronautics, Massachusetts Institute of Technology. email: cfullana@mit.edu}%
\thanks{$^{2}$Department of Electrical and Systems Engineering, University of Pennsylvania. email: luizf@seas.upenn.edu}%
\thanks{$^{3}$Department of Electrical Computer and Systems Engineering, Rensselaer Polytechnic Institute. email: paters@rpi.edu}%
}

\title{\bf Towards Safe Continuing Task Reinforcement Learning}

\begin{document}
\maketitle
\thispagestyle{empty}
\pagestyle{empty}

\begin{abstract}
Safety is a critical feature of controller design for physical systems. When designing control policies, several approaches to guarantee this aspect of autonomy have been proposed, such as robust controllers or control barrier functions. However, these solutions strongly rely on the model of the system being available to the designer. As a parallel development, reinforcement learning provides model-agnostic control solutions but in general, it lacks the theoretical guarantees required for safety. Recent advances show that under mild conditions, control policies can be learned via reinforcement learning, which can be guaranteed to be safe by imposing these requirements as constraints of an optimization problem. However, to transfer from learning \emph{safety} to learning \emph{safely}, there are two hurdles that need to be overcome: (i) it has to be possible to learn the policy without having to re-initialize the system; and (ii) the rollouts of the system need to be in themselves safe. In this paper, we tackle the first issue, proposing an algorithm capable of operating in the continuing task setting without the need of restarts. We evaluate our approach in a numerical example, which shows the capabilities of the proposed approach in learning safe policies via safe exploration.
\end{abstract}

\section{Introduction}
\label{S:intro}

Safety is a fundamental requirement of systems design needed to guarantee proper operation, but also the integrity of the environment in which they operate. While its exact definition is application-dependent, it typically arises as constraints on the states in which the underlying dynamical system can be. For power systems this can be voltage/current limitations~\cite{van2007voltage} and for robotic systems, be it single- or multi-agent, some form of collision avoidance~\cite{khatib1985real}. Several control techniques have been put forward in order to satisfy these constraints, such control barrier functions~\cite{wieland2007constructive}, artificial potentials~\cite{koditschek1990robot}, or explicit state constraints, as in the context of model predictive control~(MPC)~\cite{Mayne00c}. Yet, these techniques rely on prior knowledge of the underlying dynamical system. Without such a model, guaranteeing that a policy or control law is safe becomes considerably more challenging, especially for stochastic systems.

Reinforcement learning~(RL) has become a popular solution to the problem of obtaining model-free policies for systems expressible in terms of Markov decision processes~(MDPs)~\cite{Sutton18r}. MDPs are stochastic control processes used ubiquitously to study control problems~\cite{Shreve78a}, robotic~\cite{Kober13r}, and financial systems~\cite{Rasonyi05o}. RL enables policies to be learned by sampling trajectories rather than having direct access to the underlying MDP~\cite{Sutton18r}. While effective, it provides no guarantee that the learned policy is safe~\cite{Heger94c, Mihatsch02r, Geibel05r}. A common approach to deal with this problem is to explicitly incorporate constraints to the RL formulation~\cite{Geibel06r, Delage10p, DiCastro12p, Hutter02s, Chow17r}. Recently, it was proven that this technique results in policies that are guaranteed to be safe and can be obtained using primal-dual algorithms~\cite{Paternain19c, Paternain19s}.

The current guarantees offered by constrained RL methods, however, have two limitations that hinder their use in critical applications. First, it requires the system to be reinitialized after each rollout~\cite{sutton2000policy}. While this is not an issue in the lab, real-world systems often cannot be reset safely. What is more, when deployed to live, production systems, these resets can cause costly disruptions. Enabling constrained RL for continuing tasks is therefore paramount to achieve safety. Second, while constrained RL addresses the issue of learning safe policies, it does not guarantee that the policies are themselves learned \emph{safely}, a problem often known as safe exploration~\cite{moldovan2012safe}.

This work provides guarantees for the first limitation, i.e., it develops the theoretical underpinnings necessary to enable constrained RL for continuing tasks. To do so, we leverage recent results on the strong duality of constrained RL problems in the policy space to learn in the dual domain \cite{Paternain19c}. That is, we pose the constrained RL problem as a sequence of unconstrained ones. Further, we formulate safety in a manner amenable to the constrained RL framework. More specifically, we define safety as a guarantee on the probability of remaining in the safe set at all times. An ergodic relaxation of this definition is then used, which can be used to match the safety level of the original definition~\cite{Paternain19s}. We then derive the main contribution of this paper which is to show that, under standard conditions, primal and dual gradients provide ascent and descent directions, despite the fact that the system is not reinitialized. These results show that primal-dual approaches can potentially be used to solve constrained RL problems with continuing tasks. Additionally, strong duality also provides an equivalence between regularized and constrained formulations. Explicitly, for each value of the regularization parameters (dual variables) corresponds a safety level of the optimal policy. We exploit this fact during simulations to encourage safe rollouts, converging to the optimal solution by navigating the interior of the feasible set. We illustrate the effectiveness of this approach in an obstacle-ridden navigation example. These preliminary results point towards the possibility of using online approaches not only to learn safety but also to learn safely. Due to space constraints, obtaining guarantees for this latter procedure is the focus of future work.

\section{Problem Formulation}
\label{S:problem_formulation}

In this work we consider a stochastic optimal control problem under safety constraints and we assume that the model is unknown in situations where restarting the system is not possible or it is so rare that is not possible to rely on it. Formally, let~$\ccalS$ and~$\ccalA$ be compact sets describing the states and actions of the agent, respectively. We consider a random controller\textemdash or \emph{policy}\textemdash defined as a distribution~$\pi_\theta(a|s)$ from which the agent draws actions~$a \in \ccalA$ when in state~$s \in \ccalS$. We restrict our attention to controllers parametrized by a vector~$\theta \in \mathbb{R}^d$. The action selected by the agent drives it to another state according to the transition dynamics of the system defined by the conditional probability~$\mathbb{P}_{s_t\to s_{t+1}}^{a_t}(s) := \mathbb{P}(s_{t+1}=s \mid s_t,a_t)$, for time~$t \in \mathbb{N}$, $s_t,s_{t+1} \in \ccalS$, and~$a_t \in \ccalA$. This process is assumed to satisfy the Markov property~$\mathbb{P}(s_{t+1}=s \mid (s_u,a_u),\, \forall u \leq t) = \mathbb{P}(s_{t+1}=s \mid s_t,a_t)$, hence it is denoted as a Markov Decision Process (MDP).  

The main concern of this work is to accomplish a task while learning a policy that guarantees safety. To formalize this notion, define a safe set~$\ccalS_{\texttt{safe}} \subset \ccalS$ in which the agent is required to remain. Then, we define safety as a probabilistic extension of control invariant sets \cite{khalil2002nonlinear}.
\begin{definition}\label{def_safe_policy}
	We say a policy~$\pi_\theta$ is $(1-\delta)$-safe for the set~$\ccalS_{\texttt{safe}} \subset \ccalS$ if~$\mathbb{P}\left( \bigcap_{t \geq 0} \{s_{t} \in \ccalS_{\texttt{safe}}\} \mid \pi_\theta,s_0\right) \geq 1-\delta$.
\end{definition}
In other words, a policy is safe if the trajectories it generates remain within the safe set~$\ccalS_{\texttt{safe}}$ with high probability. In addition to guaranteeing safety in general, the system is required to complete a mission that is related to a cost or reward. Each action selected by the agent results in a reward~$r:\ccalS \times \ccalA \to \mathbb{R}$ that informs it of the quality of the decision. As such, the goal of the agent is to find a parametrization~$\theta$ of the policy that maximizes the value function of the MDP, i.e., the expected value of the discounted cumulative rewards obtained along a trajectory
\begin{equation}\label{eqn_value_function_discounted}
	V_{s_0}(\theta) = \mathbb{E}_{\ba\sim\pi_{\theta}(\ba|\bs)}
		\left[\sum_{t=0}^\infty \gamma^t r(s_t,a_t)\mid s_0\right],
\end{equation}
where~$\bs$ and $\ba$ represent the trajectory of the system and~$\gamma \in (0,1)$ is the discount factor, which defines how myopic the agent is. Thus, we can pose the following constrained safety problem by combining the objective function \eqref{eqn_value_function_discounted} with Definition \ref{def_safe_policy}. Namely,
\begin{equation}\label{eqn_optimization_problem}
\begin{aligned}
	\underset{\theta \in \mathbb{R}^d}{\text{maximize}} & && 
\mathbb{E}_{\ba\sim\pi_{\theta}(\ba|\bs)}\left[\sum_{t=0}^\infty \gamma^t r(s_t,a_t)\mid s_0\right]
	\\
	\text{subject to}&
		&&\mathbb{P}\left( \bigcap_{t = 0}^{\infty} \{ s_{t} \in \ccalS_{\texttt{safe}}\} \mid \pi_\theta,s_0 \right) \geq 1-\delta
		\text{.}
\end{aligned}
\end{equation}
Maximizing functions of the form \eqref{eqn_value_function_discounted} is a well studied problem and can be done by updating the parameters of the policy using stochastic estimates of the gradient of $V_{s_0}(\theta)$. However, guaranteeing that the policy satisfies the requirement of Definition \ref{def_safe_policy} as posed by the optimization problem \eqref{eqn_optimization_problem} is a longstanding question in situations where the transition probabilities are unknown. This prevents us from establishing relationships between the parameters of the policy and the probability of remaining in a given set. A second challenge in solving this problem arises in its nonconvexity. Even if the transition probabilities were to be known, for the most common parameterizations of policies (neural networks for instance), the problem of maximizing \eqref{eqn_value_function_discounted} while guaranteeing safety as per Definition \ref{def_safe_policy} is a constrained nonconvex problem. To overcome the first difficulty several relaxations can be proposed. One can, for instance, regularize the objective in \eqref{eqn_optimization_problem} by adding a penalty every time an unsafe set is visited \cite{tamar2012policy}. This approach however, lacks the guarantees of whether the end policy is safe or not. Another possibility is to relax the safety constraints for  average constraints \cite{Paternain19s}. We exploit this idea in Section \ref{S:crl}, which allows us to guarantee that the policies learned are safe. Despite the relaxation, the problem still remains nonconvex and constrained. We propose therefore to solve the dual problem, as despite their nonconvexity, certain forms of constrained reinforcement learning problems have zero (or small) duality gap \cite{Paternain19c}, thus motivating a dual approach. In particular, the relaxation used in Section \ref{S:crl} yields one such problem. Due to the aforementioned strong duality result, in the case of systems with restarts, we can solve the problem using a stochastic primal-dual algorithm. However, in continuing tasks (without restarts), to be able to compute gradients one typically considers undiscounted problems and assumes stationarity of the MDP for every policy \cite[Chapter 13.6]{Sutton18r}. The latter is a strong assumption and the former is in fact modifying the problem of interest in the sense that it ignores the transient behaviors of the system. In this work, we do away with these assumptions, analyzing the problem in its most general form and showing that primal-dual approaches can also be used to solve continuing task reinforcement larning problems.

\section{Constrained Reinforcement Learning for Safety}
\label{S:crl}

Going back to the formulation \eqref{eqn_optimization_problem}, if the transition probabilities of the system were to be known, the problem could be solved by directly imposing constraints on the probabilities, using for instance MPC~\cite{mayne2000constrained}. However, this is not the scenario in RL problems, where the transition probabilities can only be evaluated through experience. Hence, although the safety probabilities in~\eqref{eqn_optimization_problem} can be estimated, there is no straightforward way to optimize~$\pi_\theta$ with respect to them. To overcome this difficulty, we relax the chance constraints in~\eqref{eqn_optimization_problem} in the form of the cumulative costs as in~\eqref{eqn_value_function_discounted}. Explicitly, we formulate the following constrained reinforcement learning problem
%
  \begin{align}\label{P:parametric}
	P^\star_\theta\triangleq\max_{\theta\in\mathbb{R}^d}& \hspace{-0.5ex}  &&V_{s_0}(\theta)
	\\
	\text{subject to}& \hspace{-0.5ex} &&U_{s_0}(\theta)\triangleq \mathbb{E}\left[\sum_{t=0}^\infty\gamma^t\mathds{1}\left(s_t\in\ccalS_{\texttt{safe}}\right)\big| \pi_\theta ,s_0\right] \geq c \nonumber
  \end{align}
%
where $V_{s_0}(\theta)$ is the value function defined in \eqref{eqn_value_function_discounted} and $c$ is the appropriate level that guarantees that the resulting policy is safe according to Definition \ref{def_safe_policy}.  An explicit characterization of such value is given by the following proposition.
\begin{proposition}\label{thm_epsilon_discounted}
Suppose there exist parameters~$\tilde{\theta}$ and time horizons $T$ such that the policy~$\pi_{\tilde{\theta}}$ is~$(1-\gamma^{T}(1-\gamma)\delta)$-safe for the set~$\ccalS_{\texttt{safe}}$. Then, problem~\eqref{P:parametric} with~$c = (1-\delta[1-\gamma^{T}(1-\gamma)])/(1-\gamma)$ is feasible and its solution is~$(1-\delta)$-safe for the set~$\ccalS_{\texttt{safe}}$ up to time~$T$.
\end{proposition}
\begin{proof}
See \cite{Paternain19s}.
\end{proof}
The intuition of the previous result is that policies computed solving \eqref{P:parametric} are safe for an arbitrarily long window~$t \leq T$, assumed to exist in Proposition~\ref{thm_epsilon_discounted}. Naturally, arbitrarily safe policies must exist for this to hold and Proposition~\ref{thm_epsilon_discounted} quantifies the trade-off between safety level and safety horizon in terms of the value function discount factor~$\gamma$. Proposition~\ref{thm_epsilon_discounted} establishes that the safety level of the policies can be preserved as long as safe enough parametric policies exist. Moreover, the value $c$ required in the formulation \eqref{P:parametric} can be computed computed based on the desired safety level $\delta$ and the desired safety horizon $T$.

Having reformulated the problem of interest using a constrained reinforcement learning approach, we set focus in solving it. Notice that this problem is still nonconvex and constrained. However, despite their nonconvexity problems of the form \eqref{P:parametric} have zero duality gap \cite{Paternain19c} and therefore we can solve them in the dual domain. To do so, let us define the Lagrangian associated to \eqref{P:parametric} as the following weighted combinations of the objective function and the safety constraint
\begin{equation}\label{eqn_lagrangian}
\ccalL(\theta,\lambda) = V_{s_0}(\theta) + \lambda \left(U_{s_0}(\theta)-c\right),
\end{equation}
where $\lambda\in\mathbb{R}_+$ is the Lagrange multiplier associated to the safety constraint. A possible approach to solve problem \eqref{P:parametric} is to perform primal-dual updates with properly selected step sizes \cite{bertsekas1999nonlinear}. The primal variable is updated by following the gradient of the Lagrangian with respect to the parametrization $\theta$. Namely,
\begin{align}\label{eqn_primal_gradient_1}
\theta_{k+1} &= \theta_k +\eta_\theta \nabla_\theta \ccalL(\theta_k,\lambda_k)  \nonumber\\
&= \theta_k + \eta_\theta \bigl(\nabla_\theta V_{s_0}(\theta_k) + \lambda \nabla_\theta U_{s_0}(\theta_k)\bigr)  .
\end{align}
And the Lagrange multipliers are in turn updated in the descent direction by the following expression
\begin{align}\label{eqn_dual_step}
\lambda_{k+1} &= \bigl[\lambda_k -\eta_\lambda \nabla_\lambda \ccalL(\theta_k,\lambda_k) \bigr]_+ \nonumber \\
&=\bigl[\lambda_k -\eta_\lambda \left(U_{s_0}(\theta_k) - c\right) \bigr]_+,
\end{align}
where $\left[\cdot \right]_+$ denotes the projection onto the positive orthant.
An important observation from the algorithmic perspective is that by defining the reward
\begin{equation}
r_\lambda(s,a) = r(s,a) + \lambda \mathds{1}(s \in \ccalS_{\texttt{safe}}),
\end{equation}
the Lagrangian can be in fact written as
\begin{equation}
\ccalL(\theta,\lambda) = V_{s_0}^\lambda(\theta) = \mathbb{E}\left[ \sum_{t=0}^\infty \gamma^t r_\lambda(s_t,a_t)\mid \pi_\theta, s_0\right].
\end{equation}
This allows us to update the policy (cf., \eqref{eqn_primal_gradient_1}) as if we were solving an unconstrained reinforcement learning problem, where the required step is to compute the policy gradient of the value function $V_{s_0}^{\lambda_k}$. Hence, we can rewrite \eqref{eqn_primal_gradient_1} as
\begin{equation}\label{eqn_primal_step}
\theta_{k+1} = \theta_k +\eta_\theta \nabla_\theta V_{s_0}^{\lambda_k}(\theta_k).
\end{equation}
Using the celebrated policy gradient theorem \cite{sutton2000policy}, we can write the previous gradient as
\begin{equation}\label{eqn_policy_gradient}
\nabla_\theta V_{s_0}^{\lambda_k}(\theta) = \mathbb{E}_{(s,a)\sim \mu_{s_0}}\left[ Q^{\lambda_k}(s,a) \nabla_\theta \log \pi_\theta(a|s)\mid \pi_\theta, s_0\right],
\end{equation}
where $\mu_{s_0}(s,a)$ is the following occupancy measure
\begin{equation}\label{eqn_occupancy_measure}
\mu_{s_0} (s,a) = (1-\gamma) \sum_{t=0}^\infty \gamma^t p_\pi(s_t=s,a_t=a \mid s_0)
\end{equation}
in which for convenience we denote by $p_\pi$ the probability kernel induced by policy $\pi$, and
\begin{equation}
Q^{\lambda_k}_\theta(s,a) =  \mathbb{E}\left[ \sum_{t=0}^\infty \gamma^t r_\lambda(s_t,a_t)\mid \pi_\theta, s_0=s, a_0=a\right]. \nonumber
\end{equation}
From \eqref{eqn_policy_gradient} it becomes apparent that in order compute unbiased estimates of $\nabla_\theta V_{s_0}^{\lambda_k}(\theta) $ we are required to sample from the occupancy measure $\mu_{s_0}$. However, in order to sample from this distribution we must reset the system after each rollout. Similarly, if we want to get an unbiased estimate of $U_{s_0}(\theta)$ we are required to restart the system since its value depends on the initial condition~(cf. \eqref{P:parametric}). In the case of continuing tasks, where restarts are not possible, after the $(k+1)-$th update of the policy and the multipliers, the system reaches state $s_k$ and therefore what is available are unbiased estimates of $\nabla_\theta V^{\lambda_k}_{s_k}(\theta_k) $ and $U_{s_k}(\theta_k)$.

\begin{algorithm}[t]
  \caption{Continuing Task  Algorithm}
  \label{algorithm}
\begin{algorithmic}[1]
 \renewcommand{\algorithmicrequire}{\textbf{Input:}}
 \renewcommand{\algorithmicensure}{\textbf{Output:}}
 \Require Initial state $s_0$, step size $\eta_\theta$, $\eta_\lambda$
 \State \textit{Initialize}: Dual variable $\lambda_0$
  \For {$k=0,1,\ldots$}
    \State Advance system $T\sim \text{Geom}(\gamma)$ steps to reach state $s_k$
    \State Obtain estimates $\nabla V^{\lambda_k}_{s_k}(\theta_k) $ and $U_{s_k}(\theta_k)$
    \State Update the policy \\
\quad \quad \quad $\theta_{k+1} = \theta_k +\eta_\theta \nabla_\theta V_{s_k}^{\lambda_k}(\theta_k)$
  \State Update the dual variables \\
\quad \quad \quad   $  \lambda_{k+1}= \bigl[\lambda_k-\eta_\lambda \left(U_{s_k}(\theta_k) - c\right)
\bigr]_+$
  \EndFor
 \end{algorithmic}
 \end{algorithm}

Overall, the resulting procedure takes the form shown in Algorithm \ref{algorithm}. At time $k$, the agent advances the system by $T\sim \text{Geom}(\gamma)$ steps  and thus obtains a sample $s_k$ from the occupancy measure $\mu_{s_k}$ \cite{paternain2018stochastic}. Similarly, in order to obtain estimates of $\nabla V^{\lambda_k}_{s_k}(\theta_k) $ and $U_{s_k}(\theta_k)$, the agent can do so by accumulating the rewards $r(s_t,a_t)$ and constraint satisfaction $\mathds{1}(s \in \ccalS_{\texttt{safe}})$ by performing another rollout of $T_Q\sim \text{Geom}(\gamma)$ time steps. Then, these estimates are used to update both the primal and dual variables. Algorithm \ref{algorithm} notwithstanding, for the estimates used in it to be useful towards solving the problem \eqref{P:parametric}, we need to establish that they are primal and dual improvement directions for the Lagrangian \eqref{eqn_lagrangian}. This is the subject of the following section.

\section{Improvement Directions}
\label{S:theory}

Contrary to the offline case, we cannot reinitialize the system after each rollout in continuing tasks. This poses an issue since the gradient dynamics in~\eqref{eqn_primal_gradient_1}, or more specifically, \eqref{eqn_policy_gradient}, and~\eqref{eqn_dual_step} depend on the value function evaluated for the same initial state~$s_0$. However, after the~$k$-th update, we can only obtain estimates of~$\nabla V^{\lambda_k}_{s_k}(\theta_k) $ and $U_{s_k}(\theta_k)$. Intuitively, if these quantities carry enough information about~$\nabla V^{\lambda_k}_{s_0}(\theta_k) $ and~$U_{s_0}(\theta_k)$, then they can still be used to learn a safe policy using the dual procedure put forward in the previous section. In the sequel, we establish that they are indeed ascent and descent directions respectively for the Lagrangian~\eqref{eqn_lagrangian}. To do so, define the discounted occupation measure
\begin{equation}\label{E:occupation}
	\rho_z(s) = (1-\gamma) \sum_{t=0}^\infty \gamma^t p_\pi(s_t=s \mid s_0=z)
		\text{.}
\end{equation}
In contrast to~\eqref{eqn_occupancy_measure}, \eqref{E:occupation} is a measure over only the state space~$\calS$. The main result of this section is collected next.

\begin{theorem}\label{T:main}

Define the signed measure~$\Delta_{z,z^\prime} = \rho_{z} - \rho_{z^\prime}$ for~$z,z^\prime \in \calS$ in terms of the occupation measure~\eqref{E:occupation} and let~$\norm{\Delta_{z,z^\prime}}_\textup{TV}$ be its total variation norm. Then,
\begin{subequations}\label{E:inner_products}
\begin{align}
	\left\langle \nabla V^{\lambda_k}_{s_0}(\theta_k), \nabla V^{\lambda_k}_{s_k}(\theta_k) \right\rangle &\geq
		\norm{\nabla V^{\lambda_k}_{s_0}(\theta_k)} \times{}
	\notag\\
	&\hspace{-23mm} \left( \norm{\nabla V^{\lambda_k}_{s_0}(\theta_k)}
		- 2 \norm{D_{\lambda_k}}_{\infty,2} \norm{\Delta_{s_0,s_k}}_\textup{TV} \right)
		\label{E:inner_product_V}
	\\
	\left\langle U_{s_0}(\theta_k), U_{s_k}(\theta_k) \right \rangle
		&\geq U_{s_0}(\theta_k) \times{}
	\notag\\
	&\hspace{1.5mm} \left( U_{s_0}(\theta_k) - 2 \norm{\Delta_{s_0,s_k}}_\textup{TV} \right)
		\label{E:inner_product_U}
\end{align}
\end{subequations}
where~$\norm{D_{\lambda_k}}_{\infty,2} = \sqrt{\sum_{i = 1}^d \big( \sup_{s \in \calS} \abs{[D_{\lambda_k}(s)]_i} \big)^2}$ and
\begin{equation}\label{E:D_function}
	D_{\lambda_k}(s) = \int Q^{\lambda_k}(s,a) \nabla_\theta \pi(a \mid s) da
		\text{.}
\end{equation}
\end{theorem}
Theorem~\ref{T:main} states that the inner products in~\eqref{E:inner_products} are non-negative as long as~$\norm{\Delta_{z,z^\prime}}_\textup{TV}$ is small for all~$z,z^\prime \in \calS$. In other words, if~$\rho_{s_k}$ is a similar to~$\rho_{s_0}$, then~$\nabla V^{\lambda_k}_{s_k}(\theta_k) $ and~$U_{s_k}(\theta_k)$ are positively correlated with~$\nabla V^{\lambda_k}_{s_0}(\theta_k) $ and~$U_{s_0}(\theta_k)$. Thus, they are ascent directions of~$V_{s_0}$ and~$U_{s_0}$ and can be used in the primal-dual policy update in~\eqref{eqn_primal_gradient_1}--\eqref{eqn_dual_step}. How similar the different occupation measures must be depends essentially on how ``good'' the current policy is. When far from the optimum, the value of~$\norm{\nabla V^{\lambda_k}_{s_0}(\theta_k)}$~[$U_{s_0}(\theta_k)$] is large and the occupation measures obtained from different starting points can be quite different. On the other hand, as we approach an optimal~(safe) policy, we require increasingly precise estimates of the gradient to continue making progress, i.e., the value of~$\norm{\Delta_{s_0,s_k}}_\textup{TV}$ must be considerably smaller. We show later that this distance can be controlled by the discount factor~$\gamma$ and provide results relating it to certain mixing properties of the underlying MDP.
\begin{proof}
The proof of Theorem~\ref{T:main} stems from the next lemma.
\begin{lemma}\label{T:technical_lemma}
	Let~$H^2 = \norm{\int R(s) \rho_{z}(s) ds}^2$ for some function~$R(s): \calS \to \setR^d$ and occupation measure as~\eqref{E:occupation}. Then,
	\begin{align}
		q &= \int R^T(s) R(s^\prime) \rho_{z}(s) \rho_{z^\prime}(s^\prime) ds ds^\prime
		\notag\\
		{}&\geq H \left( H - 2 \norm{R}_{\infty,2} \norm{\Delta_{z^\prime,z}}_\textup{TV} \right)
			\label{E:integral}
			\text{.}
	\end{align}
\end{lemma}

\begin{proof}
	Notice that~\eqref{E:integral} can be rearranged to read
	\begin{equation*}
		q = H^2 + \left[ \int R(s) \rho_{z}(s) ds \right]^T
			\left[ \int R(s^\prime)\Delta_{z^\prime,z}(s^\prime) ds^\prime \right]
		\text{,}
	\end{equation*}
	which can be bounded using Cauchy-Schwartz as
	\begin{equation}\label{E:integral2}
		q \geq H^2 - H \norm{\int R(s^\prime)\Delta_{z^\prime,z}(s^\prime) ds^\prime}
		\text{.}
	\end{equation}
	We bound the remaining norm in~\eqref{E:integral2} using H\"{o}lder's inequality. Explicitly,
	\begin{align*}
		\norm{\int R(s^\prime)\Delta_{z^\prime,z}(s^\prime) ds^\prime} &\leq
			\sqrt{\sum_{i = 1}^d \sup_{s \in \calS} \abs{[R(s)]_i}^2 \left( \norm{\Delta_{z^\prime,z}}_1 \right)^2}
		\\
		{}&\leq	\norm{R}_{\infty,2} \norm{\Delta_{z^\prime,z}}_1
		\text{.}
	\end{align*}
	Combining this bound with~\eqref{E:integral2} and using the fact that~$\norm{\Delta}_1 = 2 \norm{\Delta}_\text{TV}$~(see, e.g., \cite{Levin17m}) yields~\eqref{E:integral}.
\end{proof}

To see this is the case, notice that the transition probability in the occupancy measure~\eqref{eqn_occupancy_measure} separates as in
\begin{multline*}
	p_\pi(s_t = s, a_t = a \mid s_0) ={}
	\\
	p_\pi(s_t = s \mid a_t = a, s_0) \pi(a_t = a \mid s_t = s, s_0)
	\\
	{}= p_\pi(s_t=s \mid s_0) \pi(a_t=a \mid s_t = s)
		\text{,}
\end{multline*}
where we used the Markovian property and the fact that the action depends only on the current state. Using this identity in~\eqref{P:parametric} and~\eqref{eqn_policy_gradient}, the gradient~$\nabla V^{\lambda_k}_{z}(\theta) $ and the constraint value function~$U_{z}(\theta)$ can then be written as
\begin{align}
	\nabla_\theta V_{z}^{\lambda_k}(\theta) &= \int D_{\lambda_k}(s) \rho_z(s) ds
	\\
	U_{z}^{\lambda_k}(\theta) &= \int \mathds{1}\left(s \in \ccalS \right) \rho_z(s) ds
\end{align}
for any initial state~$z \in \calS$ and~$D_{\lambda_k}$ defined as in~\eqref{E:D_function}. Hence, the inner products in~\eqref{E:inner_products} can be expanded as
\begin{align*}
	&\left\langle \nabla V^{\lambda_k}_{z}(\theta_k), \nabla V^{\lambda_k}_{z^\prime}(\theta_k) \right \rangle
	\\
	&\qquad\qquad{}= \int D_{\lambda_k}^T(s) D_{\lambda_k}(s^\prime) \rho_z(s) \rho_{z^\prime}(s^\prime) ds ds^\prime
	\\
	&\left\langle U_{z}(\theta_k), U_{z^\prime}(\theta_k) \right \rangle
	\\
	&\qquad\qquad{}= \int \mathds{1}\left(s \in \ccalS \right) \mathds{1}\left(s^\prime \in \ccalS \right)
		\rho_z(s) \rho_{z^\prime}(s^\prime) ds ds^\prime
\end{align*}
and taking~$R(s) = D_{\lambda_k}(s)$ in Lemma~\ref{T:technical_lemma} yields~\eqref{E:inner_product_V} and~$R(s) = \mathds{1}(s \in \calS)$ yields~\eqref{E:inner_product_U}.
\end{proof}

While Theorem~\ref{T:main} provides an explicit bound on the admissible difference between occupation measures starting from different states~$\norm{\Delta_{s_0,s_k}}_\textup{TV}$, it does not inform \emph{when} these occupation measures are similar. Using a different argument, \cite{paternain2020policy} shows that the inner products in~\eqref{E:inner_products} can be made positive by restricting the policy parametrization. In contrast, using the occupation measure argument from Theorem~\ref{T:main}, we provide conditions of the underlying dynamical system and control problem~(discount factor~$\gamma$) that are independent of the parametrization used for the policy. Namely, we provide two propositions showing that if the MDP ``mixes'' fast enough, then we can expect occupation measures to be similar. As is typical to obtain such results, we assume from now on that the MDP is ergodic, i.e., that the Markov chains induced by the transition probabilities~$p_\pi(s_t \mid s_{t-1})$ are ergodic for all allowed policies~$\pi$.

\begin{proposition}\label{T:tau_bound}
	Assume that the MDP is ergodic and let~$\fkp_\pi$ be its stationary distribution under the policy~$\pi$. Define~$\tau$ to be the worst-case mixing time of the MDP, i.e., $\tau = \min \{t \mid \norm{p_\pi(s_t = s \mid s_0 = z) - \fkp_\pi}_\textup{TV} \leq 1/4 \textup{ for all } z \in \calS \textup{ and } \pi \}$. Then, for $\epsilon > 1/4$,
	\begin{equation}\label{E:tau_bound}
		\gamma \geq \left( \frac{4(1-\epsilon)}{3} \right)^{1/\tau}
		\Rightarrow
		\norm{\Delta}_\textup{TV} \leq \epsilon
	\end{equation}
\end{proposition}

\begin{proof}
	Using the triangle inequality, we obtain that
	\begin{equation}\label{E:triangle}
		\norm{\Delta}_\text{TV} \leq \norm{\rho_{z} - \fkp}_\text{TV} + \norm{\rho_{z^\prime} - \fkp}_\text{TV}
			\text{,}
	\end{equation}
	where~$\fkp$ is the stationary distribution of the MDP. Note that, as with the occupation measure, we have omitted its dependence on the policy~$\pi$. Since the ergodic MDP assumption, $\fkp$ is independent of the initial state, so that using the definition of occupation measure from~\eqref{E:occupation} we get
	%
		\begin{align}
			\rho_{z}(s) &- \fkp(s) = (1-\gamma) \sum_{t = 0}^{\tau-1} \gamma^t \left[ p_\pi(s_t = s \mid s_0 = z) - \fkp(s) \right]
			\nonumber \\
			{}&+ (1-\gamma) \sum_{t = \tau}^\infty \gamma^t \left[ p_\pi(s_t = s \mid s_0 = z) - \fkp(s) \right]
				\text{.} \label{E:signed_meas}
		\end{align}
	%
	Using the triangle inequality, we bound~\eqref{E:signed_meas} by
	%
		\begin{align}
			&\|\rho_{z} - \fkp\|_\text{TV} \leq (1-\gamma) \sum_{t = 0}^{\tau-1} \gamma^t
				\norm{p_\pi(s_t = s \mid s_0 = z) - \fkp(s)}_\text{TV}
			\nonumber\\
			{}&+ (1-\gamma) \sum_{t = \tau}^\infty \gamma^t \norm{p_\pi(s_t = s \mid s_0 = z) - \fkp(s)}_\text{TV}
				\text{.}\label{E:signed_meas_tv}
		\end{align}
	%
	The first term in~\eqref{E:signed_meas_tv} can be bounded using the fact that~$\norm{p_\pi(s_t = s \mid s_0 = z) - \fkp(s)}_\text{TV} \leq 1$ and the second term, using the definition of mixing time. Evaluating the resulting series then yields
	\begin{equation}\label{E:stationary_bound}
		\norm{\rho_{z} - \fkp}_\text{TV} \leq 1 - \frac{3}{4} \gamma^\tau
			\text{,}
	\end{equation}
	which holds for all initial states~$z$. To conclude, notice that using the discount factor from~\eqref{E:tau_bound} in~\eqref{E:stationary_bound} yields~$\norm{\rho_{z} - \fkp}_\text{TV} \leq \epsilon/2$ for all~$z$, which together with~\eqref{E:triangle} implies that~$\norm{\Delta}_\text{TV} \leq \epsilon$.
\end{proof}

Proposition~\ref{T:tau_bound} establishes a relation between the discount factor and the mixing time of the MDP in order to guarantee that estimates of value functions and their gradients have a given correlation. The longer the mixing time, i.e., the slower the MDP approaches its stationary distribution, the larger the discount factor has to be for the inner products in~\eqref{E:inner_products} to be large. For instance, if the mixing time is~$\tau = 50$, $\gamma \approx 0.99$ in order to get~$\norm{\Delta}_\text{TV} \leq 0.5$. This implies an exponential window of roughly~$100$ steps or twice the mixing time.

However theoretically appealing, mixing times are hard to estimate in practice except in very particular cases. Hence, we next consider discrete, finite MDPs for which a bound can be obtained based on spectral properties of the transition probability matrix. In what follows, we consider an MDP to be reversible with respect to its stationary distribution~$\fkp_\pi$ if~$\fkp_\pi(z) p_\pi(s_t = z^\prime \mid s_0 = z) = \fkp_\pi(z^\prime) p_\pi(s_t = z \mid s_0 = z^\prime)$ for all allowed policies~$\pi$.

\begin{proposition}\label{T:lambda_bound}
	Consider a discrete, ergodic, reversible MDP. Let~$\fkp_\text{min} = \min_{s \in \calS, \pi} \fkp_\pi(s)$ be the smallest value of the stationary distribution~$\fkp_\pi$ achieved under any policy~$\pi$ and~$\lambda_\star < 1$ be an upper bound on the value of the second largest eigenvalue of the transition probability matrix~$P = [p_\pi(s_t = j \mid s_0 = i)]_{ij}$ for any policy~$\pi$. Then,
	\begin{equation}\label{E:lambda_bound}
		\gamma \geq \frac{1-\fkp_\text{min} \epsilon}{1-\lambda_\star \fkp_\text{min} \epsilon}
			\Rightarrow \norm{\Delta}_\textup{TV} \leq \epsilon
			\text{.}
	\end{equation}
\end{proposition}

\begin{proof}
	We start using the triangle inequality as in~\eqref{E:triangle} to once again reduce the problem to bounding~$\norm{\rho_{z} - \fkp}_\text{TV}$, where~$\fkp$ is the stationary distribution of the MDP. Note, once more, that since the MDP is ergodic, $\fkp$ is independent of the initial state. Together with the reversibility of the MDP, the Perron-Frobenius theorem implies that the transition probability kernel can be decomposed as
	\begin{equation}\label{key}
		p(s_t = s \mid s_0 = z, \pi) = \fkp(s) \left[ 1 + \sum_{i = 2}^n \lambda_i^t q_i(z) q_i(s) \right]
	\end{equation}
	where the~$\{q_i\}$ form an orthonormal basis of~$\setR^n$ with respect to the inner product~$\langle q,r \rangle_\fkp = \sum_{s \in \calS} q(s) r(s) \fkp(s)$~\cite[Lemma~12.2]{Levin17m}. Hence, it holds that
	\begin{equation}\label{E:signed_meas_perron}
		\rho_{z}(s) - \fkp(s) = (1-\gamma) \sum_{t = 0}^{\infty} \gamma^t
			\sum_{i = 2}^n \lambda_i^t q_i(z) q_i(s) \fkp(s)
			\text{.}
	\end{equation}
	Since both series are convergent~(by hypothesis, otherwise~$\rho_{z}(s)$ would not be a proper probability measure), the summations in~\eqref{E:signed_meas_perron} can be rearranged to get
	\begin{equation}\label{E:signed_meas_perron2}
	\begin{aligned}
		\rho_{z}(s) - \fkp(s) &= (1-\gamma) \sum_{i = 2}^n \sum_{t = 0}^{\infty} (\lambda_i \gamma)^t q_i(z) q_i(s) \fkp(s)
		\\
		{}&= \sum_{i = 2}^n \left( \frac{1-\gamma}{1 - \lambda_i \gamma} \right) q_i(z) q_i(s) \fkp(s)
			\text{.}
	\end{aligned}
	\end{equation}
	We now proceed to evaluate the total variation norm of~\eqref{E:signed_meas_perron2}. To do so, recall that the total variation norm is essentially the~$\ell_1$-norm, up to a~$1/2$ scaling. Hence, we begin by bounding the absolute value of~\eqref{E:signed_meas_perron2} as in
	\begin{equation*}
		\abs{\rho_{z}(s) - \fkp(s)} \leq
			\left( \frac{1-\gamma}{1 - \lambda_\star \gamma} \right) \fkp(s) \abs{\sum_{i = 2}^n q_i(z) q_i(s)}
			\text{,}
	\end{equation*}
	where~$\lambda_\star < 1$ is an upper bound on the value of the second largest eigenvalue of the transition probability matrix. Using Cauchy-Schwartz then yields
	\begin{multline*}
		\abs{\rho_{z}(s) - \fkp(s)} \leq{}
		\\
		\left( \frac{1-\gamma}{1 - \lambda_\star \gamma} \right) \fkp(s)
		\left[ \sum_{i = 2}^n q_i(z)^2 \sum_{i = 2}^n q_i(s)^2 \right]^{1/2}
			\text{.}
	\end{multline*}
	which can be rearranged as
	\begin{multline}\label{E:signed_meas_perron_abs}
		\abs{\rho_{z}(s) - \fkp(s)} \leq \left( \frac{1-\gamma}{1 - \lambda_\star \gamma} \right) \sqrt{\frac{\fkp(s)}{\fkp(z)}}
			\times{}
		\\
		{}\left[ \sum_{i = 2}^n q_i(z)^2 \fkp(z) \right]^{1/2} \left[ \sum_{i = 2}^n q_i(s)^2 \fkp(s) \right]^{1/2}
			\text{.}
	\end{multline}
	Recalling that the~$q_i$ are orthonormal with respect to the inner product~$\langle \cdot,\cdot \rangle_\fkp$, i.e., $\sum_{i = 1}^n q_i(s)^2 \fkp(s) = 1$, we get that
	\begin{equation}\label{E:signed_meas_perron_abs2}
		\abs{\rho_{z}(s) - \fkp(s)} \leq \left( \frac{1-\gamma}{1 - \lambda_\star \gamma} \right) \sqrt{\frac{\fkp(s)}{\fkp(z)}}
			\text{.}
	\end{equation}
	The desired total variation norm then evaluates to
	\begin{align*}
		\norm{\rho_{z} - \fkp}_\text{TV} &= \frac{1}{2} \sum_{s \in \calS} \abs{\rho_{z}(s) - \fkp(s)}
		\\
		{}&\leq \frac{1}{2} \left( \frac{1-\gamma}{1 - \lambda_\star \gamma} \right)
			\sum_{s \in \calS} \fkp(s) \frac{1}{\sqrt{\fkp(s) \fkp(z)}}
			\text{.}
	\end{align*}
	Using H\"{o}lder's inequality and the fact that~$\sum_{s \in \calS} \fkp(s) = 1$ and~$\fkp_\text{min} \leq \sqrt{\fkp(s) \fkp(z)}$ for any~$s,z \in \calS$ yields
	\begin{equation}\label{E:signed_meas_perron_tv}
		\norm{\rho_{z} - \fkp}_\text{TV} \leq \left( \frac{1-\gamma}{1 - \lambda_\star \gamma} \right)
			\frac{1}{2 \fkp_\text{min}}
			\text{.}
	\end{equation}
	For~$\gamma$ as in~\eqref{E:lambda_bound}, \eqref{E:signed_meas_perron_tv} yields~$\norm{\rho_{z} - \fkp}_\text{TV} \leq \epsilon/2$ for all~$z \in \calS$, which together with~\eqref{E:triangle} yields the desired result.
\end{proof}

Proposition~\ref{T:lambda_bound} provides a bound on the discount factor~$\gamma$ as a function of the stationary distribution of the MDP and spectral properties of its probability kernel. Assuming that~$\fkp_\text{min}$ is not too small, which is the case, e.g., when exploring the MDP using a Gaussian policy with moderate variance, the correlation between the value functions and its gradients is dominated by the rate of the slowest decaying mode of the MDP, i.e., $\lambda_\star$. Hence, if the MDP has fast decaying modes, which will typically be the case for exploration policies that have large randomness, then the value of the inner products in~\eqref{E:inner_products} will be large. Ultimately, this result is to be expected given that the relaxation time, i.e., $1/\lambda_\star$, is closely related to the mixing time of lazy Markov chains~\cite{Levin17m}. Still, despite its dependence on~$\fkp_\text{min}$ which can be quite small, the rate of~\eqref{E:lambda_bound} is more favorable than that of~\eqref{E:tau_bound}.

Before proceeding, we note that as quantitative results, neither propositions in this section are fully satisfactory. One relies on mixing times, which are hard to estimate, while the other relies on the minimum state occupation probability, which can be very close to zero, e.g., in safety applications where parts of the state space must be avoided by design. Nevertheless, they provide intuition as to when we can expect the value functions and their gradients to be correlated when starting from different points of the state space. More importantly, we have shown that the estimates of Algorithm~\ref{algorithm} provide primal ascent and dual descent directions for the Lagrangian~\eqref{eqn_lagrangian} of the constrained reinforcement learning problem \eqref{P:parametric}. This result can be utilized, together with the strong duality of the problem~\cite{Paternain19c}, to provide guarantees on the convergence of the continuing task Algorithm~\ref{algorithm}, by following an analysis along the lines of \cite{Paternain19s}.

\section{Numerical Results}
\label{S:sims}

\begin{figure}[t]
	\centering
	\includegraphics[scale=1]{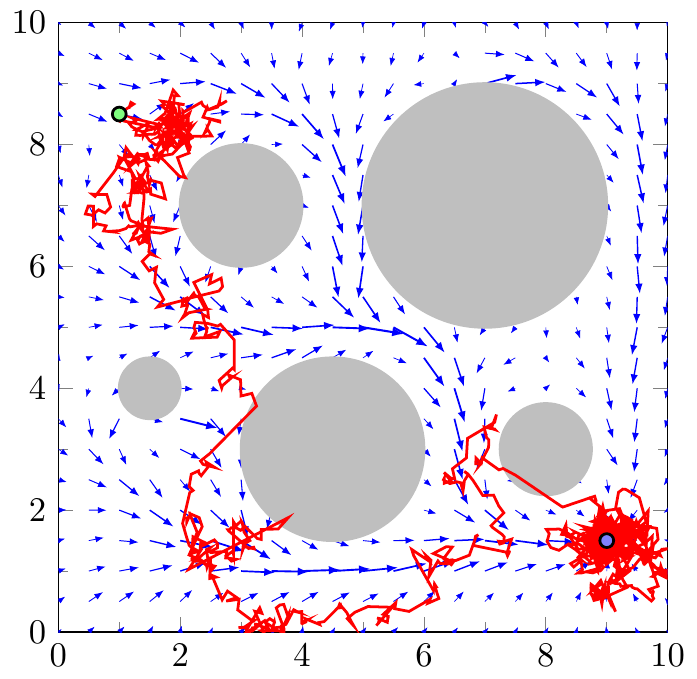} 	
	\caption{Online trajectory followed by the agent (red) and optimal offline safe navigation policy (blue). The agent starts from the position $(1,8.5)$ and navigates to the goal $(9,1.5)$ while avoiding the obstacles. A total of $2{,}000$ time steps are illustrated.}
	\label{fig:quiver}
\end{figure}

In the previous sections, we have proposed an approach to learn safely in continuing task reinforcement learning. Now, we dedicate this section to its numerical evaluation. We consider a continuous navigation problem in which an agent must traverse an environment populated with obstacles. The agent does not have any a priori knowledge of the environment and utilizes the proposed approach to navigate to a goal. The problem scenario is illustrated in Fig. \ref{fig:quiver}. Formally, we consider a MDP with state space of $\mathcal{S}=[0,10] \times [0,10]$, where the state of the agent at time $s_t \in \mathcal{S}$ is given by its position on the x- and y-axis. The dynamics of the MDP are given by the state evolution $s_{t+1}=s_{t}+T_s a_{t}$, where $T_s$ corresponds to the sampling time of the dynamics, selected to be $T_s=0.05$. Furthermore, the actions $a_t$ of the agents are governed by a Gaussian policy following
\begin{align}
\pi_\theta(a|s)=\frac{1}{\sqrt{(2 \pi)^d |\Sigma|}} 
e^{-\frac{1}{2}
\bigl(a - \mu_\theta(s) \bigr)^\top \Sigma^{-1} \bigl(a - \mu_\theta(s) \bigr)
 },
\end{align}
with covariance $\Sigma= \diag \left( 0.5,0.5\right)$. The mean of the distribution $\mu_\theta(s)$ is given by a function approximator, corresponding to a linear combination of radial basis functions
\begin{align}
\mu_\theta(s)=\sum_{i=1}^{d}  \theta_i 	\exp \left(-\frac{\|s - \bar{s}_i \|^2}{2\sigma^2} \right)
\end{align}
where $\sigma$ is the bandwidth of each kernel centered at $\bar{s}_i$ and $\theta=[\theta_1,\ldots,\theta_d]^\top$ is the vector of parameters to be learned. This linear basis is constructed by spacing at $0.25$ units each kernel of bandwidth $\sigma=0.5$. The agent attempts to complete a navigation task, for which is receives a reward $r(s,a)=-\|s - s_{\texttt{GOAL}}\|^2$, where $s_{\texttt{GOAL}}=(9,1)$ and the unsafe region corresponds to the circular obstacles shown in Fig. \ref{fig:quiver}. We follow the primal-dual approach of Algorithm \ref{algorithm}, where the required safety level is $1-\delta=0.99$, the discount factor is $\gamma=0.95$ and the step sizes are chosen to be $\eta_{\theta}=0.01$ and $\eta_{\lambda}=0.005$. In order to induce safe exploration, we initialize the dual variable to a large value of $\lambda_0=20$.

First, observe in Fig. \ref{fig:quiver} the trajectory taken by the agent. This figure shows in red the online trajectory taken by the agent, shown in contrast to the optimal policy that can be learned offline (blue). The agent manages to navigate to its goal while mostly avoiding the obstacles. However, notice that the online trajectory it follows differs from the optimal trajectory that an agent trained offline would take. This is due to the lack of a priori knowledge of the map and the randomness of the agent's policy, which it utilizes for exploration. Furthermore, notice that the agent is conservative in approaching the obstacles. Due to the large starting value of $\lambda$, when the agent enters an unsafe region, it gets pulled away considerably from that area. The occurrence of these unsafe events can be evaluated more clearly in Figure \ref{fig:safety}.

\begin{figure}[t]
	\centering
	\includegraphics[scale=1]{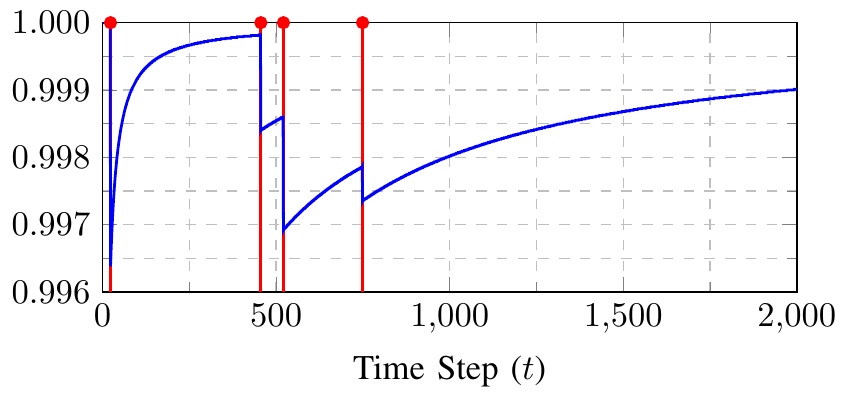} 
	\caption{Runtime safety level (blue) and unsafe events (red). The runtime safety is $\frac{1}{t}\sum_{l=0}^{t-1} \mathds{1} \bigl( s_{l} \in \ccalS_{\texttt{safe}} \bigr)$, and the unsafe events correspond to time instances in which $s_{t} \notin \ccalS_{\texttt{safe}}$. The demanded safety level is $1-\delta=0.99$.}
	\label{fig:safety}
\end{figure}

Figure \ref{fig:safety} shows the runtime safety of the agent. We also show the unsafe events, which are the instances of time in which the agents traverses unsafe regions. Overall, recall that the safety demanded of the system is $1-\delta=0.99$, and the system stays above that prespecified safety level during all its operating time, reaching the goal after around $\sim750$ time steps. Thus, the agents is capable of performing the required task in an online manner while remaining safe. Furthermore, as previously seen in Fig. \ref{fig:quiver}, unsafe events are rare due to the large value of the initial dual variable $\lambda$. These preliminary numerical results point towards the possibility of using primal-dual algorithms not only for learning safe policies but also to achieve safe exploration.

\section{Conclusions}
\label{S:conclusions}

In RL problems, two main obstacles prevent us from learning safely. First, it has to be possible to learn in an online manner, without the need of system restarts. Secondly, the system rollouts need to also be safe. In this work, we have taken a preliminary step into attaining safe continuing task reinforcement learning by providing guarantees and algorithmic details for continuous task primal-dual approaches. Thus, we have cleared one of the two major hurdles to learn safely. Furthermore, preliminary numerical results have illustrated that the primal-dual reinforcement learning approach taken in this work can potentially be used to learn safely.

\bibliographystyle{IEEEtran}
\bibliography{aux_files/safe,aux_files/IEEEabrv,aux_files/af,aux_files/bayes,aux_files/control,aux_files/gsp,aux_files/math,aux_files/ml,aux_files/rkhs,aux_files/rl,aux_files/sp,aux_files/stat}

\end{document}